\documentclass[10pt]{amsart}
\usepackage[foot]{amsaddr}
\usepackage[svgnames]{xcolor}
\usepackage{todonotes}

\usepackage{amsmath,amsfonts,mathptmx,amsthm}
\usepackage{tikz}

\usepackage[pagebackref=true]{hyperref}
\hypersetup{colorlinks = true,extension = notused,linkcolor = DarkBlue,anchorcolor = black,citecolor = Navy,filecolor = black,pagecolor = black,urlcolor = blue}

\usepackage[numbers,sort&compress]{natbib}
\usepackage{hypernat}

\usetikzlibrary{cd,arrows,shapes}
\tikzstyle{tensor}=[rectangle,thick,draw=black,fill=blue!15,minimum size=4mm]

\tikzstyle{PL}=[circle,thick,draw=black,fill=red!15,minimum size=4mm]

\tikzstyle{PR}=[circle,thick,draw=black,fill=yellow!15,minimum size=4mm]

\newcommand{\rR}{\mathbb{R}}
\newcommand{\C}{\mathbb{C}}
\newcommand{\qh}{\hat{q}}

\renewcommand{\>}{\rangle}
\newcommand{\Phr}{\mathrm{X}}
\newcommand{\Voc}{\mathrm{W}}
\newcommand{\PR}{P_{R}}
\newcommand{\PL}{P_{L}}

\newcommand{\W}{W}
\newcommand{\V}{V}

\DeclareMathOperator{\tr}{tr}
\DeclareMathOperator{\mat}{Mat}
\DeclareMathOperator{\Hom}{Hom}

\theoremstyle{plain}
\newtheorem{theorem}{Theorem}
\newtheorem{proposition}{Proposition}

\theoremstyle{remark}
\newtheorem*{remark*}{Remark}
\theoremstyle{definition}

\begin{document}
\title{Language as a matrix product state}

\author{Vasily Pestun$^1$}
\address{$^1$ Institut des Hautes \'Etudes Scientifiques (IH\'ES), Bures-sur-Yvette, France}
\email{pestun@ihes.fr}

\author{John Terilla$^2$}
\address{$^2$ Queens College and CUNY Graduate Center\\
        The City University of New York\\ New York, NY}
\email{jterilla@gc.cuny.edu}

\author{Yiannis Vlassopoulos$^1$}
\email{yvlassop@ihes.fr}

\begin{abstract}
We propose a statistical model for natural language that begins by considering language as a monoid, then representing it in complex matrices with a compatible translation invariant probability measure. We interpret the probability measure as arising via the Born rule from a translation invariant matrix product state.
\end{abstract}

\maketitle

\tableofcontents

\section{Introduction}  
Statistical language modelling, whose aim is to capture the joint probability distribution of sequences of words, has applications to problems including information retrieval, speech recognition, artificial intelligence,  human-machine interfaces, translation, and natural language problems that involve incomplete information.  
Early successes of statistical language models in industry include next-word prediction and vector embeddings of words based on colocation with reasonable performance on word similiarity exams.  Efforts to build on early successes encounter difficulties arising from the high-dimensionality of the data---the number of meaningful texts in a lanugage is exponentially smaller than the number of texts that a room full of randomly typing monkeys could produce \cite{GOODMAN2001403}.  One approach to address ``the curse of high-dimensionality''  is to truncate sequences under consideration to finite length phrases, or $n$-grams, and employ a hidden Markov model.  Since hidden Markov models essentially cutoff correlations between words beyond a fixed distance, the success of this approach depends on the application.  For example, $3$-gram and $4$-gram models have been employed effectively for speech recognition and translation, applications in which very long distance correlations are of limited importance \cite{brown2001products}.  
However, $n$-gram models belie the \emph{essential critical behavior} inherent in language \cite{zipf1949human}.  Human languages, like many biological systems including families of proteins, genomes, neurons in the brain, and the notes in a musical symphony, have significant long-range correlations that decay with a power law \cite{barbieri2012criticality,mora2011biological,2017arXiv170604432D}.  In contrast, any Markov or hidden-Markov system, such as an $n$-gram model, has long range correlations that decay exponentially.

Recently, long-short term memory (LSTM) recurrent neural networks have been employed to produce statistical language model applications that considerately outperform those based on hidden Markov models.  Notably, Google's Neural Machine Translation System \cite{DBLP:journals/corr/WuSCLNMKCGMKSJL16} and the technology in Google Voice \cite{DBLP:journals/corr/SakSB14} have advanced the state of the art in both translation and speech recognition.  While much is unknown about how these deep networks operate, new research indicates that it is hard to match
long range and higher order statistics of natural languages even with LSTM recurrent neural networks \cite{2016arXiv160606737L}.  For certain applications such as translating a few phrases, matching higher order statistics may not be very important, but for other artificial intelligence applications, such as machine determination of humorous or deceptive narratives, they are essential.

In order to develop a statistical language model capable of capturing the higher order statistics of language, we turn to quantum statistical physics, which contains models such as solvable lattice models that exhibit correlation functions that decay with the power law---the same kind of critical behavior as biological systems.  Unlike classical statistical physics, one spacial dimension suffices to exhibit criticality in quantum statistical physics \cite{korepin2004}.  So, even a one dimensional quantum statistical language model could be a better alternative to LSTM recurrent neural networks, which may be likened to \emph{classical} statistical physics.  Entanglement in a quantum many body system is the metaphorical vehicle for statistical correlation in language, and serves as the proposed method to attack the high-dimensionality of the data.  The number of basis states in a quantum many body system makes the state-space far too large to work with entirely but the number of physically relevant states occupy a subspace of exponentially smaller dimension, accessible by restricting to low-energy.   

In this paper, we introduce a simple translation-invariant quantum statistical language model on a one-dimensional lattice which we call a \emph{trace-density model}.    While this trace-density model isn't critical,  it is an experimental first step toward a critical quantum model of language.   The model involves matrix-product-states (MPS) which can approximate power--law decays ``over quite long distances'' \cite{2016arXiv160505775M}.   Two cubic constraintes are introduced.  These constraints are physically motivated and mathematically related to the moment map of the action of the unitary group $U(d)$ on complex $d\times n$ matrices.   Training algorithms based on maximizing entropy and minimizing energy, similar to what's described in \cite{pmlr-v20-bailly11,2017arXiv170901662H} can be developed and density matrix renormalization  \cite{2016arXiv160505775M} can be used.  The model is introduced first using representation theoretic language and then interpretted using the physical language of MPS.  A related language model based on an isometric tensor network is considered in \cite{tensorlang}.

\subsection{Acknowledgements} The authors would like to thank Maxim Kontsevich  and Miles Stoudenmire for helpful and stimulating discussions.  The research of V.P. on this project has received funding from the European Research Council (ERC) under the European Union's Horizon 2020 research and innovation program (QUASIFT grant agreement 677368); J.T. was supported in part by a grant from the US Army, Research, Development and Engineering Command, Mathematical Sciences Division, 711567-MA-II; and Y.V. received funding from Simons Foundation Award 385577.

\section{The trace-density model for language}
A mathematical model called a \emph{trace-density model} for a corpus of text will be described in three steps.  First, the input structures related to a corpus of text will be explained and some notation will be fixed.  Second, the  concept of  \emph{a trace-density representation} for a corpus will be defined.  Third, the property that a representation must possess in order to be considered a trace-density \emph{model} will be given.

\subsection{The structures in a corpus of text}
A corpus of text \[C=w_{i_1} w_{i_2} \cdots w_{i_M}\] is a finite sequence of words from a vocabulary $\{w_1, \ldots, w_n\}=:\Voc(C)$.  We refer to the elements of the vocabulary as words, and denote typical words with the letter $w$, but the vocabulary can be comprised of any symbols (letters, bits, ...) representing the atomic elements from which the corpus is constructed as a sequence.  Let $\Phr(C)$ denote the set of all phrases contained in the corpus $C$, a phrase being a finite subsequence consisting of adjacent elements.  The set $\Phr(C)$ of phrases has more structure than merely a set.  It is \emph{graded}---each phrase has a well defined length given by the number of words comprising that phrase and the set of phrases is the disjoint union of sets phrases with same word length: \[\Phr(C)=\bigsqcup_{k=1}^{M} \Phr^k(C) \text{ where }\Phr^k(C):=\{\text{phrases in the corpus $C$ consisting of $k$ words}.\}\]  Also, $\Phr(C)$ has a (partially defined) product $\Phr(C) \times \Phr(C) \to \Phr(C)$ defined by concatanation of phrases.  By adding a formal zero phrase to each graded component of $\Phr(C)$, the product can be extended to all of $\Phr(C)$ by defining the product of two phrases to be zero if the concatanation is not a phrase contained in the corpus.  The product is compatible with the grading
\[\Phr^k(C)\times \Phr^l(C) \to \Phr^{k+l}(C).\] Viewed with the structure, the set of phrases $\Phr(C)$ is a graded monoid, a quotient, in fact, of the free monoid generated by the vocabulary $\Voc(C)=\Phr^1(C)$. 

Moreover, for each phrase length $k$, there is a probability distribution, call it $\qh_k$, on the set $\Phr^k(C)$.  Explicitly, $\qh_k(x)$ is the number of times the phrase $x=w_{i_1}\cdots w_{i_k}$ appears in the corpus $C$, divided by $|\Phr^k(C)|$, the total number of phrases of length $k$.  If the corpus is sufficiently large, this probability distribution is considered an approximation to a nonexistent, idealized probability distribution on all phrases in the language, of which the corpus is an observed sample.  The goal is to model the collection of these idealized probability distributions.

\subsection{Trace density representations}
A \emph{density} on a Hilbert space is a positive semi-definite operator
(in the scope of this paper we do not assume that density
is normalized by unit trace). Here, we fix a finite dimension $d$ and work entirely with $\C^d$ with its standard inner product so that operators are identified with matrices.
A $d\times d$ density $P$ defines a nonnegative real valued function on the set $\mat_{d\times d}(\C)$ of $d\times d$ complex matrices by $M\mapsto \tr(MP M^*)$.  Here, the superscript $*$ denotes complex conjugate transpose  and $\tr$ denotes trace.  The $(i,i)$ element of $MP M^*$ is the nonegative real number $\<v,Pv\>$ where $v\in \C^d$ is the $i$-th row of the matrix $A$, and $\tr(MP M^*)$ is the sum of these numbers as $v$ ranges over the columns of $A$.

A \emph{$d$-dimensional trace density representation of a corpus $C$} consists of the following data:
\begin{enumerate}
\item A pair of densities $\PL$, $\PR$ on $\mathbb{C}^{d}$ such that $\tr \PL \PR = 1$
\item a function $D:\Voc(C)\to \mat_{d\times d}(\C)$ called the \emph{dictionary}. \end{enumerate}
The dictionary $D$ assigns a $d\times d$ complex matrix $M_i$ to each word $w_i$ in the vocabulary and extends to a function $\Phr(C)\to M_{d\times d}(\C)$ on all phrases by mapping a phrase to the product of matrices assigned to the words that comprise the phrase.  That is, the phrase $x=w_{i_1}\cdots w_{i_k} \in \Phr^k(C)$ is mapped to the matrix $M=M_{i_1}\cdots M_{i_k}$.  The \emph{trace density} of the trace density representation is the nonegative, real valued function $q:\Phr(C) \to \rR$ defined as the composition:
\[
\begin{tikzcd}
	  \Phr(C)	\ar{rr}{}
	  	\ar[bend left]{rrr}{q}
	& &M_{d\times d}(\C) \ar{r}
	& \rR \\[-12pt]
	w_{i_1}\ldots w_{i_k} \arrow[rr,mapsto] && M_{i_1}\cdots M_{i_k} \arrow[r,mapsto] & \tr\left(\PL \left(M_{i_1} \cdots M_{i_k}\right) \PR \left(M_{i_1} \cdots M_{i_k}\right)^*\right)
\end{tikzcd}
\]

Notice that $ \tr (P_L M P_R M^{*}) \geq 0$ because for any
two positive semi-definite operators $P, Q$ it holds that $\tr (PQ) \geq 0$
\footnote{Indeed, for positive semi-definite operators $P,Q$ let $A$ and $B$ be
Hermtian operators such that $P = A^2$, $Q = B^2$, then $\tr PQ = \tr AABB = \tr BAAB = \tr (AB)^{*} AB$}, and $M P_{R} M^{*}$ is positive semi-definite operator.

\subsection{The trace-density model of a corpus of text}
A trace density representation for a corpus of text $C$ will be considered a \emph{trace density model} for that corpus if the trace density of a phrase approximates the probability of the phrase $x$ appearing; that is, if $q(x) \approx \qh_k(x)$ for all phrases $x\in \Phr^k(C)$.  The nature of the approximation and its dependence on $k$ is left vague.

A trace density model for a corpus $C$ captures the joint probability distribution of sequences of words efficiently since all of the information is contained in the dicitonary $D$ that translates individual words to $d\times d$ matrices.  The model requires no additional memory to store the matrices assigned to phrases, sentences, paragraphs, etc...  Only a relatively efficient computation is required: a matrix product that is polynomial in the dimension of the representation and linear in the number of words.

Note also that a trace density model is translation invariant.  The probability $q(x)$ of the phrase $x$ appearing is independent of its position in the corpus.

\subsection{Graphical language for tensor networks}
Fix a corpus $C$ with vocabulary $\Voc(C)=\{w_1, \ldots, w_n\}$ and let $\W$ denote the complex $n$-dimensional vector space generated by the vocabulary.  Let $\V=\C^d$.  The dictionary $\{w_i \mapsto M_i\}$ of a $d$-dimensional trace-density representation can be assembled into a single map $M:\W \to \Hom(\V, \V)$ by extending the assignment $w_i \mapsto M_i$  linearly.  
The single map $M$ is described by $nd^2$ complex numbers $\{M_{i a b}\}$ and defines a tensor of \emph{order} $(n,d,d)$.  
A particular  number $M_{iab}$ and the entire tensor $M$ can be denoted graphically using a tensor network diagram as
 
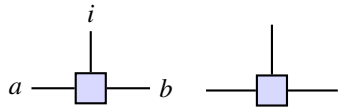
\begin{center}
    \begin{tikzpicture}
        \node[tensor] (m) at (0,0) {};
        \node[] (l) at (-1,0) {$a$};
        \node[] (r) at (1,0) {$b$};
        \node[] (i) at (0,1) {$i$};
        \draw [thick] (m) -- (l);  
        \draw [thick] (m) -- (r);
        \draw [thick] (m) -- (i);
    \end{tikzpicture}
    \begin{tikzpicture}
        \node[tensor] (m) at (0,0) {};
        \node[] (l) at (-1,0) {};
        \node[] (r) at (1,0) {};
        \node[] (b) at (0,1) {};
        \draw [thick] (m) -- (l);  
        \draw [thick] (m) -- (r);
        \draw [thick] (m) -- (b);
    \end{tikzpicture}
\end{center}
These tensor network diagrams are oriented and rotating a labeled diagram upside down indicates complex conjugation.  Connected edges denote contraction of indices.    In this pictorial language,
the following diagram represents the product of matrices $M_{i_1}M_{i_2}\cdots M_{i_k}$ 
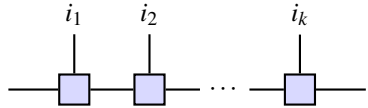
\begin{center}
    \begin{tikzpicture}
        \node[tensor] (m1) at (0,0) {};
        \node[tensor] (m2) at (1,0) {};
        \node[] (dots) at (2,0) {$\cdots$};
        \node[tensor] (mk) at (3,0) {};
        \node[] (l) at (-1,0) {};
        \node[] (r) at (4,0) {};
        \node[] (i1) at (0,1) {$i_1$};
        \node[] (i2) at (1,1) {$i_2$};
        \node[] (ik) at (3,1) {$i_k$};
        \draw [thick] (l) -- (m1) -- (m2) -- (dots) -- (mk) -- (r);  
        \draw [thick] (m1) -- (i1);
        \draw [thick] (m2) -- (i2);
        \draw [thick] (mk) -- (ik);
    \end{tikzpicture}
\end{center}
and the number $q(w_{i_1}\cdots w_{i_k})=\tr\left( \PL\left( M_{i_1} \cdots M_{i_k}\right) \PR \left(M_{i_1} \cdots M_{i_k}\right)^*\right)$  is depicted
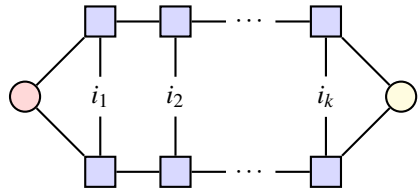
\begin{center}
    \begin{tikzpicture}
        \node[tensor] (m1) at (0,0) {};
        \node[tensor] (m2) at (1,0) {};
        \node[] (dots) at (2,0) {$\cdots$};
        \node[tensor] (mk) at (3,0) {};
        \node[PR] (pr) at (4,1) {};
        \node[PL] (pl) at (-1,1) {};
        \node[] (i1) at (0,1) {$i_1$};
        \node[] (i2) at (1,1) {$i_2$};
        \node[] (ik) at (3,1) {$i_k$};
  	    \node[tensor] (m1d) at (0,2) {};
        \node[tensor] (m2d) at (1,2) {};
        \node[] (dotsd) at (2,2) {$\cdots$};
        \node[tensor] (mkd) at (3,2) {};
        \draw [thick] (m1) -- (m2) -- (dots) -- (mk) -- (pr) -- (mkd) -- (dotsd) -- (m2d) -- (m1d) -- (pl) -- (m1);
        \draw [thick] (m1d) -- (i1) -- (m1);
        \draw [thick] (m2d) -- (i2) -- (m2);
        \draw [thick] (mkd) -- (ik) -- (mk);
    \end{tikzpicture}
\end{center}
where the circular nodes depict the densities $\PL,\PR$.  The presence of the particular indices $i_1, \dots, i_k$ indicates that they are not summed over.  The condition that $\tr(\PL \PR)=1$ is pictured as
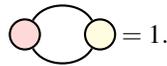
\begin{center}
    \begin{tikzpicture}
        \node[PR] (pr) at (1,1) {};
        \node[PL] (pl) at (0,1) {};
        \node[] (equal) at (1.60,1) {$=1.$};
        \draw [thick] (pl) to  [bend left=60]  (pr); \draw [thick] (pl) to [bend right=60]  (pr);
    \end{tikzpicture}
\end{center}


\section{Density and identity constraints}

We now describe a pair of technically important constraints on a trace-density representation called the \emph{left density constraint} and the \emph{right
  density constraint}.   These constraints guarantee that a trace-density has the abstract properties required of the joint-probability distributions on phrases in a language. 
  These constraints also fit into a physical interpretation of our model which we describe in Section \ref{quantum}.

\subsection{The right density constraint}
Let  $C$ be a corpus of text with vocabulary $\Voc(C)=\{w_1, \ldots, w_n\}$ of $n$ words.  Consider a $d$-dimensional trace-density representation of $C$ with density $P$ and dictionary $D$ that maps $w_i \mapsto M_i$.  The trace-density representation satisfies \emph{the right density constraint} provided
\begin{equation}
\label{densityconstraint}
\sum_{i=1}^n M_i \PR M_i^* = \PR
\end{equation}
In tensor network notation, the density constraint is 
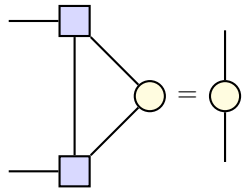
\begin{center}
    \begin{tikzpicture}
        \node[tensor] (m1) at (0,0) {};
        \node[PR] (pr) at (1,1) {};
        \node[tensor] (m1d) at (0,2) {};
        \node[] (l) at (-1,0) {};
        \node[] (ld) at (-1,2) {};
        \draw [thick] (ld) -- (m1d) -- (pr) -- (m1) -- (l);
        \draw [thick] (m1d) -- (m1);
        \node[] (equal) at (1.5,1) {$=$};
        \node[] (r) at (2,0) {};
        \node[] (rd) at (2,2) {};
        \node[PR] (pr2) at (2,1) {};
        \draw [thick] (rd) -- (pr2) -- (r);
    \end{tikzpicture}
\end{center}
The trace-density $q(x)$ of a  phrase $x$ is a nonnegative real number, but with no further a priori restrictions.  The density constraint \eqref{densityconstraint}, however, implies that the trace-density $q$ of a representation gives rise to probability distributions on the set of length $k$ phrases $\Phr^k(C)$.   
\begin{proposition}\label{thm1}
If $q$ is the trace density of a representation satisfying the density constraint \eqref{densityconstraint} then for every $k=1, 2, \ldots$ 
\[\sum_{x\in \Phr^k(C)}  q (x)=1.\] 
\end{proposition}
\begin{proof}
Applying trace to the density constraint says the sum of trace densities over all vocabulary words is one:
\[\sum_{x\in X^1(C)}q(x)=\sum_{i=1}^n q(w_i) = \sum_{i=1}^n \tr(\PL M_i \PR M_i^*)=\tr\left( \PL \sum_{i=1}^n M_i \PR M_i^*\right) = \tr(\PL \PR)=1.\]
Here is the corresponding picture
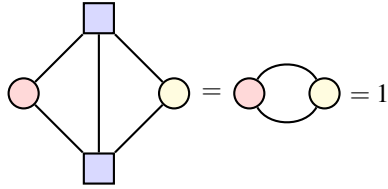
\begin{center}
    \begin{tikzpicture}
    \node[tensor] (m1) at (0,0) {};
    \node[] (equal) at (1.5,1) {$=$};
    \node[PR] (pr) at (1,1) {};
    \node[PL] (pl) at (-1,1) {};
  	\node[tensor] (m1d) at (0,2) {};
    \draw [thick] (m1d) -- (m1);
    \draw [thick] (m1d) -- (pr);
    \draw [thick] (m1) -- (pr);
    \node[PR] (pr2) at (3,1) {};
    \node[PL] (pl2) at (2,1) {};
     \node[] (equal) at (3.60,1) {$=1$};
     \draw [thick] (m1) -- (pl) -- (m1d);
     \draw [thick] (pl2) to  [bend left=60]  (pr2); 
     \draw [thick] (pl2) to [bend right=60]  (pr2);
    \end{tikzpicture}
\end{center}
which, applied repeatedly, proves the theorem:
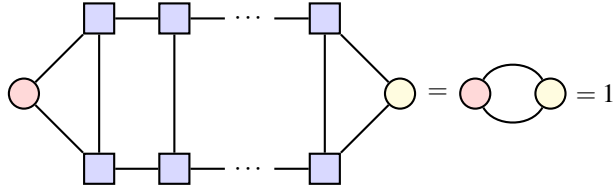
\begin{center}
    \begin{tikzpicture}
        \node[tensor] (m1) at (0,0) {};
        \node[tensor] (m2) at (1,0) {};
        \node[] (dots) at (2,0) {$\cdots$};
        \node[] (dots*) at (2,2) {$\cdots$};
        \node[tensor] (mk) at (3,0) {};
        \node[tensor] (m1*) at (0,2) {};
        \node[tensor] (m2*) at (1,2) {};
        \node[tensor] (mk*) at (3,2) {};
        \node[PR] (pr) at (4,1) {};
        \node[PL] (pl) at (-1,1) {};
        \draw [thick] (m1*) -- (m1);
        \draw [thick] (m2*) -- (m2);
        \draw [thick] (mk*) -- (mk);
        \draw [thick] (m1) -- (m2) -- (dots) -- (mk) -- (pr) -- (mk*) -- (dots*) -- (m2*) --(m1*) -- (pl) -- (m1);
        \node[] (equal) at (4.5,1) {$=$};
        \node[PR] (pr2) at (6,1) {};
        \node[PL] (pl2) at (5,1) {};
        \draw [thick] (pl2) to  [bend left=60]  (pr2); 
        \draw [thick] (pl2) to [bend right=60]  (pr2);
        \node[] (equal) at (6.6,1) {$=1$};
\end{tikzpicture}
\end{center}

\end{proof}

\subsection{The left density  constraint}
Let $C$ be a corpus of text with vocabulary $\Voc(C)=\{w_1, \ldots, w_n\}$ of $n$ words.  Consider a $d$-dimensional trace-density representation of $C$ with density $P$ and dictionary $D:w_i \mapsto M_i$.  The trace-density representation satisfies \emph{the left density constraint} provided
\begin{equation}\label{identityconstraint}
\sum_{i=1}^n M_i^* \PL M_i = \PL.
\end{equation}
%
The tensor network picture of the left density constraint is 
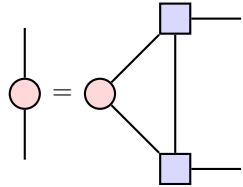
\begin{center}
    \begin{tikzpicture}
        \node[tensor] (m1) at (2,0) {};
        \node[PL] (pl) at (1,1) {};
        \node[tensor] (m1d) at (2,2) {};
        \node[] (l) at (3,0) {};
        \node[] (ld) at (3,2) {};
        \draw [thick] (ld) -- (m1d) -- (pl) -- (m1) -- (l);
        \draw [thick] (m1d) -- (m1);
        \node[] (equal) at (0.5,1) {$=$};
        \node[] (r) at (0,0) {};
        \node[] (rd) at (0,2) {};
        \node[PL] (pl2) at (0,1) {};
        \draw [thick] (rd) -- (pl2) -- (r);
    \end{tikzpicture}
\end{center}

Now, the right density constraint \eqref{densityconstraint} together with the left density constraint \eqref{identityconstraint} imply that the probability distributions defined by the trace-density fit together the way joint probabilities for sequences of words do; the probability distributions for phrases of different length are related as marginal probability distributions.
\begin{proposition}\label{thm3}
If $q$ is the trace density of a representation satisfying the density and identity constraints \eqref{densityconstraint} and \eqref{identityconstraint} and $x$ is any phrase, then for every $k,l=1, 2, \ldots $
\[q(x)=\sum_{x'\in \Phr^k(C), x'' \in \Phr^l(C)}q(x'xx'')\]
\end{proposition}
\begin{proof}
For a fixed phrase $x=w_{i_1}\cdots w_{i_s}$ , the argument begins with the picture 
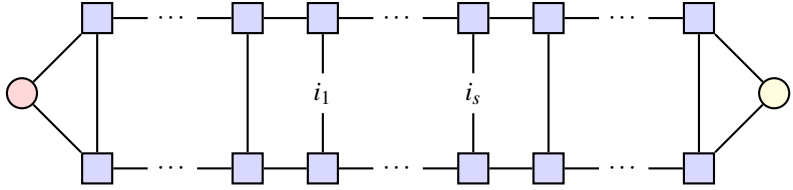
\begin{center}
    \begin{tikzpicture}
        \node[tensor] (m-1) at (-3,0) {};
        \node[] (dotsL) at (-2,0) {$\cdots$};
        \node[tensor] (m0) at (-1,0) {};
        \node[tensor] (m1) at (0,0) {};
        \node[] (dots) at (1,0) {$\cdots$};
        \node[tensor] (mk) at (2,0) {};
        \node[PR] (pr) at (6,1) {};
        \node[] (i1) at (0,1) {$i_1$};
        \node[] (ik) at (2,1) {$i_s$};
        \node[tensor] (mk1) at (3,0) {};
        \node[] (dotsR) at (4,0) {$\cdots$};
        \node[tensor] (mk2) at (5,0) {};
        \node[tensor] (mk1*) at (3,2) {};
        \node[] (dotsR*) at (4,2) {$\cdots$};
        \node[tensor] (mk2*) at (5,2) {};
        \node[PL] (pl) at (-4,1) {};

   	\node[tensor] (m-1*) at (-3,2) {};
    \node[] (dotsL*) at (-2,2) {$\cdots$};
    \node[tensor] (m0*) at (-1,2) {};
  	\node[tensor] (m1*) at (0,2) {};
    \node[] (dots*) at (1,2) {$\cdots$};
    \node[tensor] (mk*) at (2,2) {};

    \node[] (L) at (-1,0) {};
    \node[] (R) at (4,0) {};
 
  \draw [thick] (m-1*) -- (m-1);
   \draw [thick] (m0*) -- (m0);
    \draw [thick] (m1*) -- (i1) -- (m1);
    \draw [thick] (mk*) -- (ik) -- (mk);
     \draw [thick] (mk1*) --(mk1);
    \draw [thick] (mk2*) -- (mk2);
    \draw [thick] (pl) -- (m-1*) -- (dotsL*)-- (m0*) -- (m1*)  -- (dots*) -- (mk*) -- (mk1*) -- (dotsR*) -- (mk2*) -- (pr);
    \draw [thick] (pl) -- (m-1)--(dotsL) -- (m0)--(m1) -- (dots) -- (mk) -- (mk1) -- (dotsR) -- (mk2) -- (pr);

    \end{tikzpicture}
\end{center}
and repeatedly use the left density constraint to reduce the left and the right density constraint to reduce the right yielding
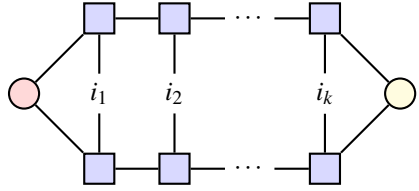
\begin{center}
    \begin{tikzpicture}
        \node[tensor] (m1) at (0,0) {};
        \node[tensor] (m2) at (1,0) {};
        \node[] (dots) at (2,0) {$\cdots$};
        \node[tensor] (mk) at (3,0) {};
        \node[PL] (pl) at (-1,1) {};
        \node[PR] (pr) at (4,1) {};
        \node[] (i1) at (0,1) {$i_1$};
        \node[] (i2) at (1,1) {$i_2$};
        \node[] (ik) at (3,1) {$i_k$};
  	    \node[tensor] (m1d) at (0,2) {};
        \node[tensor] (m2d) at (1,2) {};
        \node[] (dotsd) at (2,2) {$\cdots$};
        \node[tensor] (mkd) at (3,2) {};
        \draw [thick] (m1d) -- (i1) -- (m1);
        \draw [thick] (m2d) -- (i2) -- (m2);
        \draw [thick] (mkd) -- (ik) -- (mk);
        \draw [thick] (pl) -- (m1d) -- (m2d) -- (dotsd) -- (mkd) -- (pr) -- (mk) -- (dots) -- (m2) -- (m1) -- (pl);
    \end{tikzpicture}
\end{center}
\end{proof}
Note that the left and right density constraints are cubic  in the entries of the matrices $\PL, \PR ,M_1, \ldots, M_n$.  Yet, these two constraints imply the infinitely many higher order constraints stated in Proposition \ref{thm1} and Proposition \ref{thm3}, which are required for the joint distributions determined by the trace-density to fit together the way they must for statistical language model.

\section{Quantum physical interpretation of trace density models}\label{quantum}
This section relates a quantum physical interpretation of a trace density model for language.  Imagine a word as a quantum system consisting of a single particle having $n$ possible states---each word in the vocabulary being a possible state. Let $\W$ be the $n$-dimensional complex vector space generated by the vocabulary $\Voc(C)$. 
The space $\W$ becomes a Hilbert space with inner product defined by declaring that the vocabulary $\{w_1, \ldots, w_n\}$ defines an orthonormal, independent spanning set of basis vectors.

The Hilbert space for a quantum many body system consisting of $k$ interacting particles is $\W^{\otimes k}$ with an orthonormal basis consisting of the $n^k$ vectors 
$w_{i_1, i_2, \ldots, i_k}  := w_{i_1}\otimes w_{i_2}\otimes\cdots \otimes w_{i_k}$.  A state of such a many body system is a unit trace density $Q :\W^{\otimes k} \to \W^{\otimes k}$ and the probability that such a system is observed in the state $w_{i_1, \ldots, i_k}$ is $\tr(QO)$ where $O$ is the projection on $w_{i_1, \ldots, i_k}$.  A density $Q:\W^{\otimes k} \to \W^{\otimes k}$ induces a density on $\W^{\otimes j}\to \W^{\otimes j}$ for $j< k$ by partial trace.  

The hypothesis is that language is well described statistically by a pure state density $Q$ in a tensor product of a very large number of copies of $\W$,
which means that there exists $\psi \in W^{\otimes k}$ such that
$Q$ is projection operator on $\psi$. 

A pure state $\psi$ in a tensor product of copies of $\W$ is sometimes well approximated by a \emph{matrix product state} (MPS).  This means that there are auxillary spaces $\V_1, \ldots, \V_k$ and vectors $\phi_1 \in \W\otimes \V_1^*$, $\phi_2 \in \V_1 \otimes \W\otimes \V_2^*$, $\phi_3 \in \V_2\otimes \W\otimes \V_3^*$, ..., $\phi_{k-1} \in \V_{k-1}\otimes \W\otimes \V_k^*$, $\phi_k \in \V_k\otimes \W$ with $\psi$ obtained from $\phi_1\otimes \phi_2\otimes \cdots \otimes \phi_k$ by contracting all adjacent $\V_i^*\otimes \V_i$ pairs in the expression
\[\phi_1\otimes \phi_2 \otimes \cdots \otimes \phi_k \in \left( \W\otimes \V_1^* \right) \otimes \left( \V_1 \otimes \W \otimes \V_2^* \right) \otimes \cdots \otimes 
\left( \V_k \otimes \W\right)  \]
Note the decomposition of $\psi$ as an MPS is not unique.  Even for fixed auxillary spaces, $\V_1, \ldots, \V_k$, there is a large gauge group acting the MPS decomposition.   For example, automorphisms of each $\V_i$ act nontrivially on the MPS decomposition, while fixing the state $\psi$ obtained after contraction.


A trace-density model for language attempts to approximate the pure state $\psi$ by a translation invariant MPS.  Putting aside for the moment what happens at the far left and far right boundaries, a translation invariant MPS involves a single auxillary space $\V$ of dimension $d$ (called the bond dimension) and a single tensor $M \in \V\otimes \W \otimes \V^*$ so that $\psi$ is obtain from $M\otimes \cdots \otimes M$ by contracting $\V^*\otimes \V$ in adjacent pairs.  The space $\V \otimes \W \otimes \V^*$ is isomorphic to $\hom(\W,\hom(\V,\V))$, precisely the data of a dictionary.
 Training of translationally invariant MPS model is discussed
in  \cite{McCulloch_2008,Crosswhite_2008}. 

Now, let us consider the boundary conditions.  Assume $\psi$ is a state in an essentially infinite number of copies of $\W$.  Then, for any finite $k$, by partial trace, $\psi$ induces a state on $\W^{\otimes k}$.  Tracing out the far left and far right yields a density as pictured below
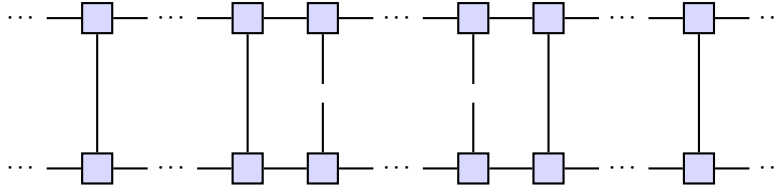
\begin{center}
    \begin{tikzpicture}

    \node[tensor] (m-1) at (-3,0) {};
    \node[] (dotsL) at (-2,0) {$\cdots$};
    \node[tensor] (m0) at (-1,0) {};
    \node[tensor] (m1) at (0,0) {};
    \node[] (dots) at (1,0) {$\cdots$};
    \node[tensor] (mk) at (2,0) {};
    \node[] (i1) at (0,1) {};
    \node[] (ik) at (2,1) {};
    \node[tensor] (mk1) at (3,0) {};
    \node[] (dotsR) at (4,0) {$\cdots$};
    \node[tensor] (mk2) at (5,0) {};
    \node[tensor] (mk1*) at (3,2) {};
    \node[] (dotsR*) at (4,2) {$\cdots$};
    \node[tensor] (mk2*) at (5,2) {};
   
   	\node[tensor] (m-1*) at (-3,2) {};
    \node[] (dotsL*) at (-2,2) {$\cdots$};
    \node[tensor] (m0*) at (-1,2) {};
  	\node[tensor] (m1*) at (0,2) {};
    \node[] (dots*) at (1,2) {$\cdots$};
    \node[tensor] (mk*) at (2,2) {};

    \node[] (L) at (-1,0) {};
     \node[] (LL) at (-4,0) {$\cdots$};
       \node[] (LL*) at (-4,2) {$\cdots$};
         \node[] (RR) at (6,0) {$\cdots$};
       \node[] (RR*) at (6,2) {$\cdots$};

  \draw [thick] (m-1*) -- (m-1);
   \draw [thick] (m0*) -- (m0);
    \draw [thick] (m1*) -- (i1) -- (m1);
    \draw [thick] (mk*) -- (ik) -- (mk);
     \draw [thick] (mk1*) --(mk1);
    \draw [thick] (mk2*) -- (mk2);
    \draw [thick] (m-1*) -- (dotsL*)-- (m0*) -- (m1*)  -- (dots*) -- (mk*) -- (mk1*) -- (dotsR*) -- (mk2*) -- (RR*);
    \draw [thick] (m-1)--(dotsL) -- (m0)--(m1) -- (dots) -- (mk) -- (mk1) -- (dotsR) -- (mk2) -- (RR);
   
    \draw [thick] (m-1) to  (LL);
     \draw [thick] (m-1*) to  (LL*);
    \end{tikzpicture}
    \end{center}
Replacing the left and right boundaries by left and right densities, one obtains 
    
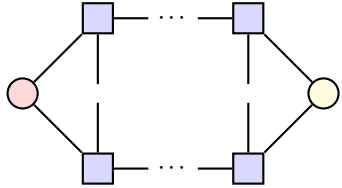
\begin{center}
    \begin{tikzpicture}
    \node[tensor] (m1) at (0,0) {};
    \node[] (dots) at (1,0) {$\cdots$};
    \node[tensor] (mk) at (2,0) {};
    \node[PR] (pr) at (3,1) {};
     \node[PL] (pL) at (-1,1) {};
    \node[] (i1) at (0,1) {};
    \node[] (ik) at (2,1) {};

  	\node[tensor] (m1d) at (0,2) {};
    \node[] (dotsd) at (1,2) {$\cdots$};
    \node[tensor] (mkd) at (2,2) {};

    \draw [thick] (m1d) -- (i1) -- (m1);
    \draw [thick] (mkd) -- (ik) -- (mk);
    \draw [thick] (m1d) -- (dotsd) -- (mkd) -- (pr);

    \draw [thick] (m1) -- (dots)-- (mk)-- (pr);
     \draw [thick] (m1) -- (pl) -- (m1*);
    \end{tikzpicture}
\end{center}


\begin{remark*}
Note that if $P_L$ is identity $I$ so that
\begin{equation}
  \sum_{i=1}^{n} M_{i}^{*} M_i = I
\end{equation}
then we can also interpret the collection $(M_i)_{i=1,\dots, n}$ as a collection of $n$ measurement operators $M_i: V \to V$
on the Hilbert space $V$ with unit trace density $\PR: V \to V$, see \cite{Nielsen_2002} page 102. 
\end{remark*}

\section{Finding trace-density models}
Trace-density representations satisfying both the right density and
the left density constraints are plentiful.  The simple trace-density representation consisting of \[\PL=I, \PR = \frac{1}{n} I,  \text{ and }M_1=M_2=\cdots =M_n=\frac{1}{\sqrt{n}} I\] shows that representations satisfying both the left density and the right density constraints exist.  To describe the moduli space of constrained representations, note that the left density constraint has the form of the isometry constraint on $M$
if the tensor $M$ is considered as a map $M: \V \to \W^{\vee} \otimes \V$ with the standard Hermitian metric on $\W$ and Hermitian metric on $\V$ defined by $\PL$.  So, for a fixed $\PL$ the space of isometric tensors $M$
(i.e. those that satisfy the left density constraint) form homogeneous space
\begin{equation}
  \frac{ U( nd )}{U(nd - d)}
\end{equation}
Moreover, there is automorphism group $U(d)$ on $\V$ preserving the Hermitian form.  Modulo action of the automorphism group the
moduli space of tensors  $M$ that satisfy the left density constraint
is Grassmanian of $d$-dimensional
complex planes in $nd$-dimensional complex space $\W^{\vee} \otimes \V \simeq \mathbb{C}^{nd}$ 
\begin{equation}
\mathrm{Gr}_{d}(\mathbb{C}^{nd}) =   \frac{ U( nd )}{U(nd - d) U(d)}
\end{equation}
Given a tensor $M$ satisfying the left identity constraint, an appropriate positive density $\PR$ that fits with the right
density constraint can be found, as the following theorem proves. 
\begin{theorem}
  Let $M$ be a tensor $M$ of order $(n,d,d)$ satisfying the
  left density constraint (\ref{identityconstraint}) with $\PL$.  Then there exists a right density $\PR$ so that $M$ together with $\PR$ satisfies the
  right density constraint.
\end{theorem}
\begin{proof}
  Let $M$ be any tensor of order $(n,d,d)$ and consider the operator $\hat{M}:
  \mathrm{Mat}_{d \times d}\to  \mathrm{Mat}_{d \times d} $ defined by $$\hat{M}(A)=\sum_{i=1}^n M_i A M_i^*.$$  

 (1) For any $A$, $\left(\hat{M}(A)\right)^* = \hat{M}\left(A^*\right)$ and so $\hat{M}$ preserves self-adjoint matrices. 

 (2)  For any $v\in \C^d$, $\langle \hat{M}(A)v,v\rangle= \sum_{i=1}^n \langle M_i^*A M_iv,v\rangle=\sum_{i=1}^n \langle A Mv,Mv\rangle$ and so $\hat{M}$ preserves positive semi-definite matrices.

 (3) We have $\tr(\PL \hat{M}(A))=\tr\left(\PL \sum_{i=1}^n M_i A M_i^*\right)=\tr\left(\sum_{i=1}^n M_i^* \PL M_i A \right) = \tr ( \PL A) $ so if $M$ satisfies the left density constraint (\ref{identityconstraint}), 
  the operator $\hat{M}$ preserves
 the hyperlane in the space of $d \times d$ matrices $A$ defined
 by the linear equation $\tr(P_{L} A) = 1$.

 Combining (1), (2) and (3) we obtain that $\hat M$ is
a map of the set $\mathcal{A}_{P_{L}}$ of positive semi-definite self-adjoint
matrices that satisfy constraint $\tr(P_L A) = 1$ to itself.
Moreover, $\hat M: \mathcal{A}_{P_{L}} \to \mathcal{A}_{P_{L}} $ is continuous because it is a linear operator. 
  
Since the set $\mathcal{A}_{P_{L}}$ is convex (a hyperplane section
of a convex set of positive semi-definite operators), it is
homeomorphic to a closed ball, then
Brauer's fixed point theorem implies that there exists a density $P_R \in \mathcal{A}_{P_{L}}$ such that  $\hat{M}(\PR)=\PR.$  
\end{proof}

The left density and right density constraints make it possible to numerically find  a trace-density model by a 
maximum log-likelihood algorithm \cite{ROSENFELD1996187}.  The idea is to find a trace-density representation that maximizes the (logarithm of the) trace-density for a training corpus.  Intuitively, the constraints make certain that the total trace density over all possible phrases of a fixed length will equal one, and so maximizing the trace-densities of the phrases in the corpus will automatically make the exponential number of nonsense phrases have nearly zero trace-density.

\bibliography{lang}{}

\providecommand{\href}[2]{#2}\begingroup\raggedright\begin{thebibliography}{10}

\bibitem{GOODMAN2001403}
J.~T. Goodman, ``A bit of progress in language modeling,'' {\em Computer Speech
  and Language} {\bf 15} (2001), no.~4 403 -- 434.

\bibitem{brown2001products}
A.~D. Brown and G.~E. Hinton, ``Products of hidden markov models.,'' in {\em
  AISTATS}, Citeseer, 2001.

\bibitem{zipf1949human}
G.~Zipf, {\em Human behavior and the principle of least effort: an introduction
  to human ecology}.
\newblock Addison-Wesley Press, 1949.

\bibitem{barbieri2012criticality}
R.~Barbieri and M.~Shimono, ``Criticality in large-scale brain fmri dynamics
  unveiled by a novel point process analysis,'' {\em Networking of
  Psychophysics, Psychology and Neurophysiology} (2012) 61.

\bibitem{mora2011biological}
T.~Mora and W.~Bialek, ``Are biological systems poised at criticality?,'' {\em
  Journal of Statistical Physics} {\bf 144} (2011), no.~2 268--302.

\bibitem{2017arXiv170604432D}
{\L}.~{D{\c e}bowski}, ``{Is Natural Language Strongly Nonergodic? A Stronger
  Theorem about Facts and Words},'' {\em ArXiv e-prints} (June, 2017)
  \href{http://xxx.lanl.gov/abs/1706.04432}{{\tt 1706.04432}}.

\bibitem{DBLP:journals/corr/WuSCLNMKCGMKSJL16}
Y.~W. et~al., ``Google's neural machine translation system: Bridging the gap
  between human and machine translation,'' {\em CoRR} {\bf abs/1609.08144}
  (2016).

\bibitem{DBLP:journals/corr/SakSB14}
H.~Sak, A.~W. Senior, and F.~Beaufays, ``Long short-term memory based recurrent
  neural network architectures for large vocabulary speech recognition,'' {\em
  CoRR} {\bf abs/1402.1128} (2014).

\bibitem{2016arXiv160606737L}
H.~W. {Lin} and M.~{Tegmark}, ``{Critical Behavior from Deep Dynamics: A Hidden
  Dimension in Natural Language},'' {\em ArXiv e-prints} (June, 2016)
  \href{http://xxx.lanl.gov/abs/1606.06737}{{\tt 1606.06737}}.

\bibitem{korepin2004}
V.~Korepin, ``Universality of entropy scaling in one dimensional gapless
  models,'' {\em Phys. Rev. Lett.} {\bf 92} (2004), no.~9.

\bibitem{2016arXiv160505775M}
E.~{Miles Stoudenmire} and D.~J. {Schwab}, ``{Supervised Learning with
  Quantum-Inspired Tensor Networks},'' {\em ArXiv e-prints} (May, 2016)
  \href{http://xxx.lanl.gov/abs/1605.05775}{{\tt 1605.05775}}.

\bibitem{pmlr-v20-bailly11}
R.~Bailly, ``Quadratic weighted automata:spectral algorithm and likelihood
  maximization,'' in {\em Proceedings of the Asian Conference on Machine
  Learning} (C.-N. Hsu and W.~S. Lee, eds.), vol.~20 of {\em Proceedings of
  Machine Learning Research}, (South Garden Hotels and Resorts, Taoyuan,
  Taiwain), pp.~147--163, PMLR, 14--15 Nov, 2011.

\bibitem{2017arXiv170901662H}
Z.-Y. {Han}, J.~{Wang}, H.~{Fan}, L.~{Wang}, and P.~{Zhang}, ``{Unsupervised
  Generative Modeling Using Matrix Product States},'' {\em ArXiv e-prints}
  (Sept., 2017) \href{http://xxx.lanl.gov/abs/1709.01662}{{\tt 1709.01662}}.

\bibitem{tensorlang}
V.~{Pestun} and Y.~{Vlassopoulos}, ``{Tensor network language model},'' {\em
  ArXiv e-prints} (Oct., 2017) \href{http://xxx.lanl.gov/abs/1710.10248}{{\tt
  1710.10248}}.

\bibitem{McCulloch_2008}
I.~P. {McCulloch}, ``{Infinite size density matrix renormalization group,
  revisited},'' {\em ArXiv e-prints} (Apr., 2008)
  \href{http://xxx.lanl.gov/abs/0804.2509}{{\tt 0804.2509}}.

\bibitem{Crosswhite_2008}
G.~M. {Crosswhite}, A.~C. {Doherty}, and G.~{Vidal}, ``{Applying matrix product
  operators to model systems with long-range interactions},'' {\em Physics
  Review B} {\bf 78} (July, 2008) 035116,
  \href{http://xxx.lanl.gov/abs/0804.2504}{{\tt 0804.2504}}.

\bibitem{Nielsen_2002}
M.~A. Nielsen and I.~Chuang, {\em Quantum computation and quantum information}.
\newblock AAPT, 2002.

\bibitem{ROSENFELD1996187}
R.~Rosenfeld, ``A maximum entropy approach to adaptive statistical language
  modelling,'' {\em Computer Speech and Language} {\bf 10} (1996), no.~3 187 --
  228.

\end{thebibliography}\endgroup
\bibliographystyle{utphys}

\end{document}